\documentclass{article}

\usepackage[preprint,nonatbib]{neurips_2020}

\usepackage[utf8]{inputenc} % allow utf-8 input
\usepackage[T1]{fontenc}    % use 8-bit T1 fonts
\usepackage{hyperref}       % hyperlinks
\usepackage{url}            % simple URL typesetting
\usepackage{booktabs}       % professional-quality tables
\usepackage{amsfonts}       % blackboard math symbols
\usepackage{nicefrac}       % compact symbols for 1/2, etc.
\usepackage{microtype}      % microtypography

\usepackage{color}
\usepackage{bbm}
\usepackage{graphicx}
\usepackage{amsmath}
\usepackage{amssymb}
\usepackage{algorithmic}        % use algorithmic
\usepackage[vlined,ruled]{algorithm2e}    % use algorithm2e
\SetKwProg{Fn}{Function}{}{}

\usepackage{amsthm}
\usepackage{enumitem}

\theoremstyle{plain}
\newtheorem{thm}{Theorem}
\theoremstyle{plain}
\newtheorem{cor}{Corollary}
\theoremstyle{definition}

\title{Hierarchical Class-Based Curriculum Loss}

\author{%
  Palash Goyal \\
  Samsung Research America\\
  Mountain View, CA 94043 \\
  \texttt{palash.goyal@samsung.com} \\
   \And
  Shalini Ghosh \\
  Samsung Research America\\
  Mountain View, CA 94043 \\
  \texttt{shalini.ghosh@samsung.com} \\
}

\begin{document}

\maketitle

\begin{abstract}
  Classification algorithms in machine learning often assume a flat label space. However, most real world data have dependencies between the labels, which can often be captured by using a hierarchy. Utilizing this relation can help develop a model capable of satisfying the dependencies and improving model accuracy and interpretability. Further, as different levels in the hierarchy correspond to different granularities, penalizing each label equally can be detrimental to model learning. In this paper, we propose a loss function, hierarchical curriculum loss, with two properties: (i) satisfy hierarchical constraints present in the label space, and (ii) provide non-uniform weights to labels based on their levels in the hierarchy, learned implicitly by the training paradigm. We theoretically show that the proposed loss function is a tighter bound of 0-1 loss compared to any other loss satisfying the hierarchical constraints. We test our loss function on real world image data sets, and show that it significantly substantially outperforms multiple baselines.
\end{abstract}

\section{Introduction}
\label{sec:intro}
Machine learning (ML) models are trained on class labels that often have an underlying taxonomy or hierarchy defined over the label space. For example, a set of images may contain objects like ``building" and ``bulldog". There exists a class-subclass relation between the ``dog" and ``bulldog" classes --- so, if the model predicts the object to be a dog instead of bulldog, a human evaluator will consider the error to be mild. In comparison, if the model predicts the object to be ``stone" or ``car", then the error will be more egregious. Although such nuances are often not visible through standard evaluation metrics, they are extremely important when deploying ML models in real world scenarios.

Hierarchical multi-label classification (HMC) methods, which utilize the hierarchy of class labels, aim to tackle the above issue. Traditional methods in this domain broadly use one of three approaches: (i) architectural modifications to the original model, to learn either levels or individual classes separately, (ii) converting the discrete label space to a continuous one embedding the labels using relations between them, and (iii) modifying the loss function adding more weights to specific classes in the hierarchy. However, the methods in this domain are mostly empirical and the choice of modifications is often experimental. To overcome this issue, we aim to incorporate the class dependencies in the loss function in a systematic fashion. To this end, we propose a formulation to incorporate hierarchical constraint in a base loss function and show that our proposed loss is a tight bound to the base loss.

Further, we note that typically humans do not learn all the categories of objects at the same time but learn them gradually starting with simple high level categories. A similar setting was explored by Bengio \emph{et. al.}, introducing the concept of curriculum learning feeding the model easier examples to mimic the way of human learning. They show that learning simple examples first makes the model learn a smoother function. Lyu \emph{et. al.} extended this to define an example-based curriculum loss with theoretical bounds to 0-1 loss. We extend our hierarchically constrained loss function to incorporate a class-based curriculum learning paradigm, implicitly providing higher weights to simpler classes. With the hierarchical constraints, the model ensures that the classes higher in the hierarchy are selected to provide training examples until the model learns to identify them correctly, before moving on to classes deeper in the hierarchy (making the learning problem  more difficult).

We theoretically show that our proposed loss function, hierarchical class-based curriculum loss, is a tight bound on 0-1 loss. We also show that any other loss function that satisfies hierarchical constraints on a given base loss gives a higher loss compared to our loss. We evaluate this result empirically on two image data sets, showing that our loss function provides a significant improvement on the hierarchical distance metric compared to the baselines. We also show that, unlike many other hierarchical multi-label classification methods, our method doesn't decrease the performance on non-hierarchical metrics and in most cases gives significant improvement over the baselines.

Overall, we make the following contributions in this paper:
\begin{enumerate}[noitemsep,topsep=0pt,parsep=0pt,partopsep=0pt]
    \item We introduce a hierarchically constrained loss function to account for hierarchical relationship between labels in a hierarchical taxonomy.
    \item We provide theoretical analysis for proving that this formulation of adding constraints renders the hierarchical loss function tightly bound w.r.t. the 0-1 loss.
    \item We add a class-based curriculum loss formulation on the constrained loss function, based on the intuition that shallower classes in the hierarchy are easier to learn compared to deeper classes in the same taxonomy path.
    \item We show that our class-based hierarchical curriculum loss renders a tighter bound to 0-1 loss by smoothing the hierarchical loss.
    \item We experimentally show the superiority of the proposed loss function compared to other state-of-the-art loss functions (e.g., cross entropy loss),  w.r.t. multiple metrics (e.g., hierarchical distance, Hit@1).
    \item We provide ablation studies on the constraints and curriculum loss to empirically show the interplay between them on multiple datasets.
\end{enumerate}

%\subsection{Organization}
The rest of the paper is organized as follows. In Section~\ref{sec:rel}, we cover the related works in this domain. Then in Section~\ref{sec:hcl}, we provide a mechanism to incorporate hierarchical constraints and show that the proposed loss is tightly bounded to the base loss. We then extend it to incorporate the class-based curriculum loss and show a tighter bound to 0-1 loss. In Section~\ref{sec:exp}, we show experimental evaluations and ablations on two real world image data sets. Finally, we conclude in Section~\ref{sec:conc} summarizing our findings and mentioning future work.

\section{Related Work}
\label{sec:rel}
Research in hierarchical classification falls into three categories: (i) changing the label space from discrete to continuous by embedding the relation between labels, (ii) making structural modifications to the base architecture as per the hierarchy, and (iii) adding hierarchical regularizers and other loss function modifications.

Label-embedding methods learn a function to map class labels to continuous vectors capable of reconstructing the relation between labels. One advantage of such methods is their generalizability to the type of relations between labels. They represent the relation between labels as a general graph and use an embedding approach~\cite{GOYAL201878} to generate a continuous vector representation. Once the label space is continuous, they use a continuous label prediction model and predict the embedding. The disadvantage is typically the difficulty of mapping back the prediction to the discrete space and the noise introduced in this conversion. In text domain, works typically use word2vec~\cite{mikolov2013efficient} and Glove~\cite{pennington2014glove} to map the words to vectors. Ghosh~\emph{et. al}~\cite{ghosh2016} use contextual knowledge to constrain LSTM-based text embeddings, Miyazaki \emph{et. al}~\cite{miyazaki2019label} use a combination of Bi-LSTMs and hierarchical attention to learn embeddings for labels. For a general domain, Kumar \emph{et. al}\cite{kumar2018multi} use maximum-margin matrix factorization to get a piece-wise linear mapping of labels from discrete to continuous space. DeViSE~\cite{frome2013devise} method maps classes to a unit hypersphere using analysis on Wikipedia text. TLSE~\cite{chen2019two} uses a two-stage label embedding method using neural factorization machine~\cite{he2017neural} to jointly project features and labels into a latent space. CLEMS~\cite{huang2017cost} propose a cost-sensitive label embedding using classic multidimensional scaling approach for manifold learning. SoftLabels~\cite{bertinetto2019making} assigns soft penalties for classes as we go farther from the ground truth label in the hierarchy.

Models which perform structural modifications use earlier layers in the network to predict higher level categories and later layers to predict lower level categories. HMCN~\cite{wehrmann2018hierarchical} defines a neural network layer for each layer in the hierarchy and uses skip connections between input and subsequent layers to ensure each layer gets prediction output of previous layer and the input. They also proposed a variant HMCN-R which uses recurrent neural network (RNN) to share parameters and show the performance only deteriorates a little using an RNN. AWX~\cite{masera2018awx} also proposes a hierarchical output layer which can be plugged at the end of any classifier to make it hierarchical. HMC-LMLP~\cite{cerri2014hierarchical} propose a local model in which they define a neural network for each term in the hierarchy. Alsallakh~\cite{bilal2017convolutional} combines hierarchical modifications using visual-analytics method and a hierarchical loss and show how capturing hierarchy can significantly improve AlexNet. Note that this class of methods in which we modify structural modifications is often domain dependent often needing a lot of time to analyze the data and come up with the modifications. In comparison, our loss based HCL method is easier to implement and domain agnostic.

Finally, models which modify the loss function to incorporate hierarchy assign a higher penalty to the prediction of labels which are more distant from the ground truth. AGGKNN-L and ADDKNN-L-CSL~\cite{verma2012learning} modify the metric space and use a lowest common ancestor (LCA) based penalty between the classes to train their model. Similarly, Deng \emph{et. al.}~\cite{deng2010does} modify the loss function to introduce hierarchical cost giving penalty based on the height of LCA. CNN-HL-LI~\cite{wu2016learning} use a weighting parameter to control the contribution of fine-grained classes which is empirically learned. HXE~\cite{bertinetto2019making} use a probabilistic approach to assign penalties for a given class given the parent class and provide an information theoretic interpretation for it. In another interesting work by Baker \emph{et. al.}~\cite{baker2017initializing}, authors initialize weights of the final layer in the neural network according to the relations between labels and show significant gains in the performance. In this paper, we propose a novel hierarchical loss function. Neural network models have also used different kinds of loss functions for particular domain applications, e.g., focal loss for object detection~\cite{lin2018focal,Li2019,zhang2020}, ranking loss to align preferences~\cite{Chen2009,selvaraju2019taking,zhang2019}.

The domain of curriculum learning was introduced by Bengio \emph{et. al.}~\cite{bengio2009curriculum} based on the observation that humans learn much faster when presented information in a meaningful order as opposed to random which is typically used for training machine learning models. They tested this idea on machine learning by providing the model with easy examples first and then subsequently providing harder examples. Several follow up works have shown this type of learning to be successful. Cyclic learning rates~\cite{smith2017cyclical} combine curriculum learning with simulated annealing to modify the learning rate between a range and achieve faster convergence. Smith \emph{et. al.}~\cite{smith2019super} further extend this to get super-convergence using large learning rates. Teacher-student curriculum learning~\cite{matiisen2019teacher} uses a teacher to assist the student in learning the tasks at which the student makes the fastest progress. Curriculum loss~\cite{lyu2019curriculum} introduces a modified loss function to select easy examples automatically. 

The above works on hierarchical multi-label classification are in general empirical and require careful study of the application domain. In this work, we propose a hierarchical class-based curriculum loss with theoretical analysis and provable bounds rendering it free of hyperparameters. Further, we propose a class-based curriculum loss to enhance the performance of hierarchically constrained loss.

\section{Hierarchical Class-Based Curriculum Loss}
\label{sec:hcl}
For the task of multi-class classification, given a multi-class loss function, our goal is to incorporate the hierarchical constraints present in the label space into the given loss function. Note that we consider the general problem of multi-label multi-class classification of which a single label multi-class classification is an instance, thus yielding our method applicable to a wide variety of problems. In this section, we first define the general formulation of multi-class multi-label classification. We then define a hierarchical constraint which we require to be satisfied for a hierarchical label space. We then introduce our formulation of a hierarchically constrained loss and show that the proposed loss function indeed satisfies the constraint. We then prove a bound on the proposed loss. We extend the loss function to implicitly use a curriculum learning paradigm and show a tight bound to 0-1 loss using this. Finally, we present our algorithm to train the model using the our loss function.
\subsection{Incorporating Hierarchical Constraints}
Consider the learning framework with training set $\mathcal{T} = {(x_i, y_i)}_{i=1,\ldots,N}$ with $N$ training examples and input image features $x_i\in \mathbb{R}^D$. We represent the labels as $y_i \in \{-1, 1\}^C$ where $C$ is the number of classes and $y_{i,j} = 1$ means that the $i^{th}$ example belongs to $j^{th}$ class. A 0-1 multi-class multi-label loss with this setting can be defined as follows:
\begin{equation}
    e(y, \hat{y}) = \sum_{i \in \mathbb{N}} \sum_{j \in \mathbb{C}} \mathbbm{1}(y_{i,j} \neq sign(\hat{y}_{i,j})).
\end{equation}

Let the set of classes $\mathbb{C}$ be arranged in a hierarchy $\mathbb{H}$ defined by a hierarchy mapping function $h: \mathbb{C} \rightarrow 2^{\mathbb{C}}$ which maps a category $c \in \mathbb{C}$ to its children categories. We use the function $m: \mathbb{C} \rightarrow \mathbb{M}$ to denote the mapping from a category $c$ to its level in the hierarchy. We now define the following hierarchical constraint $\Lambda$ on a generic loss function $l$, the satisfaction of which would yield the loss function $l_h$:
\begin{equation}
\label{eq:lambda}
    \Lambda: \forall (x_i, y_i) \in \mathcal{T}, \forall c_1, c_2 \in \mathbb{C}, m(c_1) > m(c_2) \implies l_h(y_{i,c_1}, \hat{y}_{i, c_1}) \geq l_h(y_{i,c_2}, \hat{y}_{i, c_2})
\end{equation}
The constraint implies that the loss increases monotonically with the level of the hierarchy i.e. loss of higher (i.e. closer to the root) levels in the hierarchy is lesser than that of the lower levels (i.e. closer to the leaves). The intuition is that identifying categories in higher level is easier than categories in lower level as they are coarser. Its violation would mean that the model is able to differentiate between finer-grained classes more easily compared to coarse classes which is counterintuitive. 

We now propose $e_h$, a variant of the 0-1 loss $e$, and show that it satisfies the hierarchical constraint $\Lambda$:
\begin{equation}
\label{eq:e_h}
    e_h(y, \hat{y}) = \sum_{i \in \mathbb{N}} \sum_{j \in \mathbb{C}} \max(e(y_{i,j}, \hat{y}_{i,j}), \max_{k: m(k) < m(j)}(e(y_{i,k}, \hat{y}_{i,k}))
\end{equation}

We now show the following result for $e_h$:
\begin{thm}
\textbf{(Hierarchically Constrained 0-1 Loss)} Given a 0-1 loss function $e(y, \hat{y})$, the loss function $e_h$ defined in Equation~\ref{eq:e_h} satisfies the hierarchical constraint $\Lambda$ defined in Equation~\ref{eq:lambda}.
\end{thm}
\begin{proof}
Let us assume that $e_h$ doesn't follow hierarchical constraint $\Lambda$ i.e.
\begin{equation}
\label{eq:lambdaFalse}
    \exists (x_i, y_i) \in \mathcal{T}, \exists c_1, c_2 \in \mathbb{C}\text{ s.t. } m(c_1) > m(c_2)\text{ and } e_h(y_{i,c_1}, \hat{y}_{i, c_1}) < e_h(y_{i,c_2}, \hat{y}_{i, c_2})
\end{equation}
However, we have
\begin{equation}
\begin{split}
    e_h(y_{i,c_1}, \hat{y}_{i, c_1}) &= \max(e(y_{i,c_1}, \hat{y}_{i,c_1}), \max_{k: m(c_k) < m(c_1)}(e(y_{i,c_k}, \hat{y}_{i,c_k}))\\
                                     &\geq \max(e(y_{i,c_1}, \hat{y}_{i,c_1}), \max_{k: m(c_k) \leq m(c_2)}(e(y_{i,c_k}, \hat{y}_{i,c_k}))\\
                                     &= \max(e(y_{i,c_1}, \hat{y}_{i,c_1}), e(y_{i,c_2}, \hat{y}_{i,c_2}), \max_{k: m(c_k) < m(c_2)}(e(y_{i,c_k}, \hat{y}_{i,c_k}))\\
                                     &= \max(e(y_{i,c_1}, \hat{y}_{i,c_1}), e_h(y_{i,c_2}, \hat{y}_{i,c_2}))\\
                                     &\geq e_h(y_{i,c_2}, \hat{y}_{i,c_2}),
\end{split}
\end{equation}
\end{proof}
which contradicts the assumption in Eq.~\ref{eq:lambdaFalse}. This concludes that $e_h$ follows hierarchical constraint $\Lambda$.

Now, we show that the hierarchically constrained loss function is tightly bounded to the base function:
\begin{thm}
\label{thm:constrained01}
\textbf{(Bound on Constrained 0-1 Loss)} For a 0-1 loss function $e(y, \hat{y})$, the loss function $e_h$ defined in Equation~\ref{eq:e_h} is an element-wise tight bound on $e(y, \hat{y})$ with constraint $\Lambda$. Let $\preceq$ denote elementwise inequality i.e. $f \preceq q$ means $\forall x \in domain(f), f(x) \leq q(x)$. We then have:
\begin{equation}
    e \preceq e_h \preceq g\text{ } \forall g \in \mathcal{L}\text{ satisfying hierarchical constraint $\Lambda$, s.t. }  e \preceq g.
\end{equation}
\end{thm}
\begin{proof}
Let us assume that $\exists g \in \mathcal{L}$ s.t. $g \prec e_h$ i.e. 
\begin{equation}
    \exists (x_i, y_i) \in \mathcal{T}, \exists c_1\in \mathbb{C}\text{ s.t. } g(y_{i, c_1}, \hat{y}_{i, c_1}) < e_h(y_{i, c_1}, \hat{y}_{i, c_1})
\end{equation}
As $e_h$ is a 0-1 loss function, this implies the following:
\begin{equation}
     g(y_{i, c_1}, \hat{y}_{i, c_1}) < 1\text{ and } e_h(y_{i, c_1}, \hat{y}_{i, c_1}) = 1
\end{equation}
Substituting definition of $e_h$ from equation~\ref{eq:e_h}, from the second condition we get
\begin{equation}
    \max(e(y_{i,c_1}, \hat{y}_{i,c_1}), \max_{k: m(c_k) < m(c_1)}(e(y_{i,c_k}, \hat{y}_{i,c_k}))) = 1
\end{equation}
This leads to two cases:

\textbf{Case 1:} If $e(y_{i,c_1}, \hat{y}_{i,c_1}) = 1$, then as $g(y_{i, c_1}, \hat{y}_{i, c_1}) < 1$, we have $g(y_{i, c_1}, \hat{y}_{i, c_1}) < e(y_{i,c_1}, \hat{y}_{i,c_1})$ which violates the $e \preceq g$ condition.

\textbf{Case 2:} If $\max_{k: m(c_k) < m(c_1)}(e(y_{i,c_k}, \hat{y}_{i,c_k})) = 1$, then $e(y_{i,c_2}, \hat{y}_{i,c_2}) = 1$ s.t. $m(c_2) < m(c_1)$. Since $e \preceq g$, we have $g(y_{i,c_2}, \hat{y}_{i,c_2}) \geq 1$ for $m(c_2) < m(c_1)$. However, $g(y_{i,c_1}, \hat{y}_{i,c_1}) < 1$ which violates $\Lambda$ constraint.

Thus, we get $e \preceq e_h \preceq g$.
\end{proof}

The above results can be generalized for any base loss function $l$ as follows:
\begin{cor}
\label{corr:l_h}
\textbf{(Bound on Constrained Generic Loss)} For any loss function $l(y, \hat{y})$, the loss function $l_h$ defined below is an element-wise tight bound on $l(y, \hat{y})$ with constraint $\Lambda$ i.e.
\begin{equation}
    l_h(y, \hat{y}) = \sum_{i \in \mathbb{N}} \sum_{j \in \mathbb{C}} \max(l(y_{i,j}, \hat{y}_{i,j}), \max_{k: m(c_k) < m(c_j)}(l(y_{i,k}, \hat{y}_{i,k})),
\end{equation}
the following relation holds:
\begin{equation}
    l \preceq l_h \preceq g\ \forall g \in \mathcal{L}\text{ satisfying hierarchical constraint $\Lambda$, s.t. }  l \preceq g.
\end{equation}
\end{cor}

We have shown above that the hierarchy constrained loss function $l_h$ provides an element-wise tight bound on the base loss $l$. We now extend this loss function to use a curriculum learning paradigm and show that the loss is a tighter bound to 0-1 loss compared to any other loss satisfying hierarchical constraints.

\subsection{Hierarchical Curriculum Loss}
As shown by Hu et. al.~\cite{hu2016does}, 0-1 loss ensures that the empirical risk has a monotonic relation with adversarial empirical risk. However, it is non-differentiable and difficult to optimize. Following the groundwork by Lyu et. al.~\cite{lyu2019curriculum} who propose example based curriculum loss, we present a class-based curriculum loss for any given loss function $l$ following the hierarchical constraint in the following theorem. The theorem also proves that the function defined is tighter bound to 0-1 loss compared to any loss function which satisfies the hierarchical constraint and is lower bounded by  $l$. Note that a general loss function is element-wise lower bounded by 0-1 loss i.e. $e \preceq l$.
\begin{thm}
\label{thm:hccl}
\textbf{(Hierarchical Class-Based Curriculum Loss)} For a general hierarchy constrained loss function $l_h(y, \hat{y})$, we define the loss function $l_{hc}$ as follows:
\begin{equation}
\label{eq:hcl}
    l_{hc}(y, \hat{y}) = \min_{s \in \{0, 1\}^C} \max \bigg(\sum_{i \in \mathbb{N}} \sum_{j \in \mathbb{C}}s_j l_h(y_{i,j}, \hat{y}_{i,j}), C - \sum_{j \in \mathbb{C}}s_j + e_h(y, \hat{y})\bigg)
\end{equation}
Then, $e(y, \hat{y}) \leq l_{hc}(y, \hat{y}) \leq g(y, \hat{y})\text{ } \forall g \in \mathcal{L}\text{ satisfying hierarchical constraint $\Lambda$, s.t. }  l \preceq g$ i.e. the following holds
\begin{equation}
    |l_{hc}(y, \hat{y}) - e(y, \hat{y})| \leq |g(y, \hat{y}) - e(y, \hat{y})| \forall g \in \mathcal{L} \text{ such that }
\end{equation}
\begin{equation}
    \forall (x_i, y_i) \in \mathcal{T}, \forall c_1, c_2 \in \mathbb{C}, m(c_1) > m(c_2) \implies g(y_{i,c_1}, \hat{y}_{i, c_1}) \geq g(y_{i,c_2}, \hat{y}_{i, c_2}) \text{ and }
\end{equation}
\begin{equation}
    l(y_{i,j}, \hat{y}_{i,j}) \leq g(y_{i,j}, \hat{y}_{i,j}) \forall (x_i, y_i) \in \mathcal{T}, \forall j\in \mathbb{C}
\end{equation}
\end{thm}
\begin{proof}
Consider
\begin{equation}
\begin{split}
    l_{hc}(y, \hat{y}) &= \min_{s \in \{0, 1\}^C} \max (\sum_{i \in \mathbb{N}} \sum_{j \in \mathbb{C}}s_j l_h(y_{i,j}, \hat{y}_{i,j}), C - \sum_{j \in \mathbb{C}}s_j + e_h(y, \hat{y}))\\
        & \leq \max (l_h(y, \hat{y}), e_h(y, \hat{y})) = l_h(y, \hat{y})\\
        % & = l_h(y, \hat{y})\\
        & \leq g(y, \hat{y})\text{(from Corollary~\ref{corr:l_h})}.
\end{split}
\end{equation}

For the lower bound we have,
\begin{equation}
\begin{split}
    l_{hc}(y, \hat{y}) &= \min_{s \in \{0, 1\}^C} \max (\sum_{i \in \mathbb{N}} \sum_{j \in \mathbb{C}}s_j l_h(y_{i,j}, \hat{y}_{i,j}), C - \sum_{j \in \mathbb{C}}s_j + e_h(y, \hat{y}))\\
        & \geq \min_{s \in \{0, 1\}^C}  \hat{y}_{i,j}), C - \sum_{j \in \mathbb{C}}s_j + e_h(y, \hat{y}) = e_h(y, \hat{y})\\
        % & = e_h(y, \hat{y})\\
        & \geq e(y, \hat{y})\text{(from Theorem~\ref{thm:constrained01})}.
\end{split}
\end{equation}
We thus have $e(y, \hat{y}) \leq l_{hc}(y, \hat{y}) \leq g(y, \hat{y})$. Subtracting $e(y, \hat{y})$ from both sides, we get the theorem.
\end{proof}

\subsection{Algorithm}
\begin{algorithm}[!htb]\color{black}
\Fn{selectClasses (Training Data $\mathcal{T} = {(x_i, y_i)}_{i=1,\ldots,N}$, Base Loss $l$, Threshold $thresh$)}{
 \For{$j=1 \ldots C$}{
  $l_h(y_{., j}, \hat{y}_{., j})$ $\leftarrow$ 0\;
  \For{$i=1 \ldots N$}{
    $l_h(y_{i, j}, \hat{y}_{i, j})$ $\leftarrow$ $\max(l(y_{i,j}, \hat{y}_{i,j}), \max_{k: m(c_k) < m(c_j)}(l(y_{i,k}, \hat{y}_{i,k}))$\;
    $l_h(y_{., j}, \hat{y}_{., j})$ += $l_h(y_{i, j}, \hat{y}_{i, j})$\;
  }
   
}
Get $K$ s.t. $\sum_{c=1}^{K}l_h(y_{., c}, \hat{y}_{., c}) > thresh + 1 - K$\;
 \For{$i=1 \ldots C$}{
     \If{$i$ $<$ $K$}{
         $s_i$ $\leftarrow$ 1\;
     }
     \Else{
        $s_i$ $\leftarrow$ 0\;
     }
 }
 return $s$}
\caption{Class Selection for Hierarchical Class-Based Curriculum Learning}
\label{alg:s_selection}
\end{algorithm}

In the above theorem, we prove that the proposed hierarchical class-based curriculum loss provides a tighter bound to 0-1 loss compared to the hierarchically constrained loss function. Given the above, we now need to find the optimal class selection parameters $s_i$ for each class. We show that Algorithm~\ref{alg:s_selection} provides the optimal selection parameters:
\begin{thm}
\label{thm:alg}
\textbf{(Class Selection)} Given a base loss function $l$, a hierarchically constrained loss function $l_h$, a solution $s$ for Equation~\ref{eq:hcl} is provided by Algorithm~\ref{alg:s_selection}.
\end{thm}
The proof for above is provided in the Appendix. Note that the time complexity of above algorithm is $O(NC\log(C))$ and is thus computationally inexpensive as the number of classes doesn't typically go over orders of thousands.

\section{Experiments}
\label{sec:exp}
We first perform an ablation study on each of the component of hierarchical class-based curriculum loss, including the hierarchically constrained loss and class-based curriculum loss, and show how they interplay to provide the final loss function. We then compare our loss function against the state-of-the-art losses to show its performance gain. 
\subsection{Experimental Setup}
We evaluate our loss function on two real world image data sets -- Diatoms~\cite{dimitrovski2012hierarchical} and IMCLEF~\cite{dimitrovski2011hierarchical}. Diatoms data set contains 3,119 images of diatoms (a large and ecologically important group of unicellula or colonial organisms (algae)). Each diatom can correspond to one or many of the categories arranged in a hierarchy. Overall, there are 399 categories in this data set arranged into a hierarchy of height 4 containing 47 categories. On the other hand, IMCLEF data set contains images of x-rays of human bodies and classes correspond to the body parts arranged in a hierarchy of height 4. For both these data sets, we use a pre extracted feature set, extracted using techniques of image segmentation including Fourier transforms and histograms of local SIFT descriptors.

For evaluation, we use a multi-layer perceptron with the above extracted features as input and the categories as output. We select the hyperparameters of the neural network using evaluation on a validation set with binary cross entropy loss. Based on this, we get a structure with 800 hidden neurons and a dropout of 0.25. Note that we fix this network for all the baseline loss functions and our loss function to ensure fair comparison of results. We compare the hierarchical class-based curriculum loss with the following state-of-the-art losses -- (i) binary cross entropy loss~\cite{Goodfellow-et-al-2016}, (ii) focal loss~\cite{lin2017focal} and (iii) hierarchical cross entropy loss~\cite{bertinetto2019making}. Further, we also compare it with SoftLabels~\cite{bertinetto2019making}, which modifies the ground truth labels in accordance with the hierarchy.

We use the following metrics to evaluate each of the losses for the classification task -- (i) Hit$@$1, (ii) MRR (Mean Reciprocal Rank)~\cite{radev2002evaluating}, and (iii) HierDist~\cite{deng2010does}. The first three metrics capture the accuracy of ranking of the model predictions and hierarchy capturing methods often show lower performance compared to non-hierarchical methods as the losses get more constrained. Hierarchical methods often show improvements on a metric which captures how close to the ground truth class the prediction is in the given hierarchy.

Our final metric, HierDist, captures this and can be defined as the minimum height of the lowest common ancestor (LCA) between the ground truth labels and the top prediction from the model. Mathematically, for a data point $(x_i, y_i) \in \mathcal{T}$, it can be defined as follows:
\begin{equation*}
    HierDist = \min_{c_1 \in \{j: y_{i, j} == 1\}} LCA_{\mathbb{H}}\big(c_1, argmax_{j}(\hat{y}_{i, j})\big),
\end{equation*}
where $\mathbb{H}$ denotes the hierarchy of the labels.  Note that as pointed out by Deng~\cite{deng2010does}, the metric is effectively on a log scale. It is measured by the height in the hierarchy of the lowest common ancestor, and moving up a level can more than double the number of descendants depending on the fan out of the parent class (often greater than 3-4). We show that our loss function is superior to the baseline losses for this metric. In addition, our model's performance also doesn't deteriorate on non-hierarchical metrics.
%$\min_{c_1 \in \{j: y_{i, j} == 1\}} LCA_{\mathbb{H}}(c_1, c_2)$ with $c_2 = argmax_{j}\hat{y}_{i, j}$

\subsection{Ablation Studies}
\label{sec:abl}
\begin{table}[!ht]
  \caption{Ablation studies showing the effect of each component of our loss function .}
  \label{tab:abl}
  \centering
  \begin{tabular}{lllllll}
    \toprule
    Methods & \multicolumn{3}{c}{Diatoms} & \multicolumn{3}{c}{IMCLEF}          \\
    \cmidrule(r){2-4} \cmidrule(r){5-7}
    & Hit@1 & MRR & HierDist & Hit@1 & MRR & HierDist \\
    \midrule
    CrossEntropy & 74.45 & 85.19 & 1.26 & 90.45 & 93.33 & 0.35\\
    HCL-Hier & 75.12 & 81.46 & 1.23 & 90.45 & 93.33 & 0.33\\
    HCL-CL & 74.93 & 81.34 & 1.24 & 90.65 & 93.34 & 0.32\\
    HCL & \textbf{75.21} & \textbf{81.56} & \textbf{1.22} & \textbf{90.95} & \textbf{93.49} & \textbf{0.22}\\
    \bottomrule
  \end{tabular}
\end{table}
We show the effects of the hierarchical constraints and the curriculum loss using cross entropy loss as the base function in Table~\ref{tab:abl}. The results in the data set show evaluation on a held out test data.

Consider the hierarchically constrained cross entropy loss, HCL-Hier. For Diatoms, we observe that the loss function gives improvement on the hierarchical metric significantly decreasing the hierarchical distance from 1.26 to 1.23. Further, we see that the non-hierarchical metrics also improve compared to the baseline loss. For IMCLEF, we observe that the non-hierarchical metrics stay similar to the cross entropy baseline but the loss gives improvement on the hierarchical distance metric. This shows that our formulation of adding the hierarchical constraint doesn't deteriorate typical metrics all the while making the predictions more consistent with the label space hierarchy as well as making predictions closer to the ground truth label in the hierarchical space. We believe that the improvement in non-hierarchical metrics may be attributed to the modified ranking due to hierarchical constraint making the top predictions more likely.

Looking into class-based curriculum loss individually, we observe that as suggested by theory, carefully selecting the classes based on the their loss (implicitly giving more weight to simpler classes) improves the hierarchical evaluation metric, making the predictions more relevant. Further, as this way of learning also makes the learned model function smoother, we observe that in both the data sets the non-hierarchical metrics also improve significantly with respect to baseline. In both the data sets, we observe similar trends of getting improvement in all the metrics.

Overall, combining the two aspects yields us HCL which shows the best performance. We see that for all the metrics, HCL gives significant gains both with respect to baseline loss and individual components. This follows theory in which we have show that combining class-based curriculum loss with hierarchical constraints gives a tighter bound to 0-1 loss with respect to the hierarchically constrained loss. Further, as this loss explicitly ensures that loss of a higher level node is lower than the a lower level, and implicitly gives more weight to higher level node as they are selected more, the combined effect makes it more consistent to the hierarchy. This can be particularly seen with IMCLEF data set, for which each individual component gave good improvements for the hierarchical metric but the combined loss gave much more significant gain. We also note that Diatoms has 399 classes compared to 47 classes of IMCLEF with similar number of levels making the number of categories in each layer around 8 times higher making the gains in HierDist more difficult to attain. 
% \begin{table}[!h]
%   \caption{Hierarchical Image Classification Results. We use FixEfficientNet as the base approach for image classification and vary the training loss function to compare their performance. For Diatoms and IMCLEF, we use pre-extracted features with a multi-layer perceptron.}
%   \label{image-class}
%   \centering
%   \begin{tabular}{lllllllllllll}
%     \toprule
%     Dataset & \multicolumn{2}{c}{CrossEntropy} & \multicolumn{2}{c}{FocalLoss} & \multicolumn{2}{c}{Hier-CE}  & \multicolumn{2}{c}{SoftLabels} & \multicolumn{2}{c}{HMCN}  & \multicolumn{2}{c}{HCL-CL}         \\
%     \cmidrule(r){2-3} \cmidrule(r){4-5} \cmidrule(r){6-7} \cmidrule(r){8-9} \cmidrule(r){10-11}  \cmidrule(r){12-13}
%     & MAP & Hit@1 & MAP & Hit@1 & MAP & Hit@1 & MAP & Hit@1 & MAP & Hit@1 & MAP & Hit@1 \\
%     \midrule
%     Animal\\
%     Food\\
%     Human\\
%     Landscape\\
%     Man-made\\
%     iNaturalist\\
%     Diatoms & 73.0 & 74.6 &  &  &  &  &  &  & &  & 73.5 & 75.8\\
%     IMCLEF & 82.5 & 90.6 &  &  &  &  &  &  & &  & 83.2 & 91.0\\
%     \bottomrule
%   \end{tabular}
% \end{table}

\subsection{Results}
\begin{table}[!ht]
  \caption{Hierarchical Image Classification Results on Diatoms and IMCLEF data sets. We use pre-extracted features with a multi-layer perceptron as our base model.}
  \label{tab:res}
  \centering
  \begin{tabular}{lllllll}
    \toprule
    Methods & \multicolumn{3}{c}{Diatoms} & \multicolumn{3}{c}{IMCLEF}          \\
    \cmidrule(r){2-4} \cmidrule(r){5-7}
    & Hit@1 & MRR & HierDist & Hit@1 & MRR & HierDist \\
    \midrule
    CrossEntropy & 74.45 & 81.00 & 1.26 & 90.45  & 93.33 & 0.35\\
    FocalLoss & 74.36 & 80.47 & 1.27 & 90.85  & 93.60 & 0.26\\
    Hier-CE & 75.12 & 81.51 & 1.24 & 90.85  & \textbf{93.63} & 0.24\\
    SoftLabels & 72.36 & 74.95 & 1.38 & 90.45  & 92.07 & 0.33\\
    % HMCN \\
    HCL & \textbf{75.21} & \textbf{81.56} & \textbf{1.22} & \textbf{90.95} & 93.49 & \textbf{0.22}\\
    \bottomrule
  \end{tabular}
\end{table}
We now compare our proposed loss (HCL) with the state-of-the-art loss functions capturing hierarchy as well as a label embedding method (SoftLabels). From Table~\ref{tab:res}, we observe that our loss significantly outperforms the base loss functions. Consistent with the results presented by Bertinetto \emph{et al.}\cite{bertinetto2019making}, we see that previously proposed hierarchically constrained loss functions especially SoftLabels improve on hierarchical metrics but the performance deteriorates for non-hierarchical metrics. On the other hand, HCL's performance improves or stays comparable to baselines on non-hierarchical metrics, while getting significant gains on the hierarchical metric over state-of-the-art loss functions. We observe that Hier-CE is the best performing baseline but our model outperforms it on every metric except MRR on IMCLEF. As pointed above, the gains in Diatoms are more difficult to obtain given the number of classes in each level but our model gives visible gains on it as well.

\section{Conclusion}
\label{sec:conc}
In this paper, we proposed a novel loss function for multi-label multi-class classification task ensuring the predictions are consistent with the hierarchical constraints present in the label space. 
%We provided two components of the loss which can individually be used for a classification task. 
For hierarchical loss, we proposed a formulation to modify any loss function to incorporate hierarchical constraint and show that this loss function is a tight bound on the base loss compared to any other loss satisfying hierarchical constraints. We next proposed a class-based curriculum loss, which implicitly gives more weight to simpler classes rendering the model smoother and more consistent with the hierarchy as shown experimentally. We combined the two components and theoretically showed that the combination provides a tighter bound to 0-1 loss, rendering it more robust and accurate. We validated our theory with experiments on two multi-label multi-class hierarchical image data sets of human x-rays and diatoms. We showed that the results are consistent with the theory and show significant gains on the hierarchical metric for the data sets. Further, we observed that our models also improve on non-hierarchical metrics making the the loss function more widely applicable.
In the future, we would like to relax the hierarchical constraints and develop a loss for a general graph based relation structure between the labels. Further, we would like to the test the model on other real world data sets which contain relation between labels. Finally, we would also like to test the model performance when we introduce noise into the hierarchy and the labels.

% \newpage
\bibliographystyle{plain}
\bibliography{neurips_2020}

\end{document}